%% file: main.tex
\documentclass[conference]{IEEEtran}
\IEEEoverridecommandlockouts    
\usepackage[letterpaper]{geometry}
\usepackage{amsthm}
\usepackage{multirow}
\usepackage{lipsum}
\usepackage{graphicx}
\usepackage{epstopdf}
\usepackage{amsmath,amsfonts,amssymb}
\usepackage{algorithmic}
\usepackage[ruled,vlined]{algorithm2e}
\usepackage{marvosym}
\usepackage{mathrsfs}
\usepackage{cite}
\usepackage{array}
\usepackage{soul}
\usepackage{color}
\usepackage{booktabs}
\usepackage{graphics} 
\usepackage{epsfig} 
\usepackage{mathptmx} 
\usepackage{times} 
\usepackage{amsmath} 
\usepackage{amssymb}  
\usepackage{hyperref}
\usepackage{float}
\usepackage{multirow}
\usepackage{paralist, tabularx}
\usepackage[dvipsnames]{xcolor}
\usepackage{soul}
\usepackage[normalem]{ulem}

\newtheorem{definition}{Definition}
\newtheorem{assm}{Assumption}
\newtheorem{lemma}{Lemma}
\title{\LARGE \bf
FORMULA: FORmation MPC with neUral barrier Learning for safety Assurance

}

\author{%
Qintong Xie$^{1,*}$, Weishu Zhan$^{2,*}$ and Peter Chin$^{1}$%
\thanks{$^*$ These authors contributed equally to this work.}%
\thanks{Qintong Xie and Peter Chin are with Thayer School of Engineering, Dartmouth College, NH 03784.}%
\thanks{Weishu Zhan is with the Department of Computer Science, The University of Manchester, Manchester M13 9PL, United Kingdom.}%
}

\begin{document}

\maketitle

\thispagestyle{empty}
\pagestyle{empty}

\begin{abstract}
Multi-robot systems (MRS) are essential for large-scale applications such as disaster response, material transport, and warehouse logistics, yet ensuring robust, safety-aware formation control in cluttered and dynamic environments remains a major challenge. Existing model predictive control (MPC) approaches suffer from limitations in scalability and provable safety, while control barrier functions (CBFs), though principled for safety enforcement, are difficult to handcraft for large-scale nonlinear systems. This paper presents \textbf{FORMULA}, a safe distributed, learning-enhanced predictive control framework that integrates MPC with Control Lyapunov Functions (CLFs) for stability and neural network-based CBFs for decentralized safety, eliminating manual safety constraint design. This scheme maintains formation integrity during obstacle avoidance, resolves deadlocks in dense configurations, and reduces online computational load. Simulation results demonstrate that FORMULA enables scalable, safety-aware, formation-preserving navigation for multi-robot teams in complex environments.


\end{abstract}

\input{section/Introduction}

\input{section/Problem_Formulation}
\input{section/CONTROLLER_DESIGN}

\input{section/Experiments}

\input{section/Conclusion_and_Future_Work}

%
%
%
%
\bibliography{main}
\bibliographystyle{IEEEtran}
%
%

\end{document}

%% file: section/introduction.tex
\section{Introduction}

Multi-robot systems (MRS) are essential for handling complex, large-scale tasks such as disaster response \cite{tian2020search, ghamry2017multiple}, material transport \cite{pang2012agent}, and warehouse management \cite{li2023double, kattepur2018distributed}.
However, keeping stable formations is challenging due to terrain variability and dynamic obstacles.
Despite extensive research \cite{grover2023before, wang2017safety, zhang2025gcbf+}, developing a robust, scalable controller for safe formation maintenance in uncertain conditions still requires further study.

Traditional approaches, such as leader-follower formations \cite{xiao2016formation} and virtual structures \cite{zhou2018agile}, perform well in structured environments but struggle in highly dynamic scenarios. Ren et al. \cite{ren2007information} proposed a consensus-based approach that effectively maintains formation but lacks independent obstacle avoidance capabilities. Zhou et al. \cite{zhou2018agile} utilized artificial potential fields (APFs), modeling obstacles as high-potential regions that generate repulsive forces to steer agents away and ensure collision-free trajectories. However, their reliance on local interactions can lead to deadlocks and limited trajectory optimization. To improve adaptability, hybrid methods have been proposed, such as Qi et al. \cite{qi2022formation}, who integrated consensus theory with pigeon-inspired obstacle avoidance, and approaches combining deep reinforcement learning with stochastic braking \cite{zhao2021usv}. RL approaches require extensive training data and struggle to enforce safety, limiting real-world applicability.

MPC is a powerful framework for multi-robot coordination, balancing formation control and obstacle avoidance via constrained optimization. Stability is often enforced through Control Lyapunov Functions (CLFs), yielding MPC-CLF schemes that enhance robustness and performance even with imperfect models \cite{grandia2020nonlinear, minniti2021adaptive}. Distributed MPC (DMPC) has been applied to platooning under unidirectional communication \cite{zheng2016distributed} and nonconvex trajectory optimization via alternating direction methods \cite{ferranti2022distributed}, while Lyapunov-based DMPC further improves disturbance rejection \cite{wei2019distributed}. However, computational complexity remains a key limitation for large-scale systems, as nonlinear solvers incur high costs for long horizons on resource-constrained platforms. Explicit MPC mitigates this by precomputing control laws offline \cite{alonso2020explicit}, improving real-time performance at the expense of exponential offline complexity and strict model fidelity. More recently, learning-based policy refinement has emerged as a way to adapt control policies online, alleviating computational burdens while approximately preserving closed-loop optimality.

Control Barrier Functions (CBFs) provide a complementary, model-based route to safety. Centralized CBF frameworks have been developed for simple dynamics \cite{ames2019control}, and extensions to polynomial systems using Sum-of-Squares optimization broaden theoretical guarantees \cite{xu2017correctness}, albeit with limited scalability. Distributed CBFs leverage local observations and neighbor communication \cite{wang2017safety, zhang2023neural}, but nonlinear dynamics and input constraints remain challenging. Learning-based CBFs improve adaptability \cite{qin2021learning}, yet can overly conflate cooperative agents with obstacles, leading to conservative behaviors \cite{dawson2023safe}. Hybrid methods combine NMPC with CBFs for dynamic feasibility \cite{thirugnanam2022safety}, and predictive CBF filters improve safety in learning-based control \cite{wabersich2022predictive}. However, CBF-based systems are prone to deadlocks; disturbance-based and rotational strategies have been explored to mitigate stagnation \cite{wang2017safety, grover2023before}. Overall, guaranteeing strict safety under input constraints while preserving real-time scalability in large-scale MRS remains an open challenge \cite{chen2020guaranteed, agrawal2021safe}. 

Our framework integrates MPC-CLF for stability and a neural network-based CBF for decentralized safety, ensuring real-time feasibility with formal safety guarantees. Traditional methods lack explicit obstacle avoidance, encounter deadlocks, or impose high computational costs, while our approach enforces safety constraints and ensures motion feasibility in dense formations. By explicitly guaranteeing low formation error, our approach provides a scalable, efficient, and robust solution for formation control in complex environments. Our contributions are as follows:

\begin{itemize}
\item We introduce FORMULA, a distributed predictive control framework for multi-robot formation control and collision avoidance, equipped with an explicit deadlock resolution mechanism.

\item We provide guarantees on formation integrity during obstacle avoidance, ensuring safety-aware navigation and preservation of the desired formation, independent of system scale or agent count.

\item We replace hand-tuned safety constraints with neural network-based control barrier functions, improving scalability and enabling rapid adaptation to new environments via retraining, while preserving formal safety guarantees.
\end{itemize}

%% file: section/Problem_Formulation.tex
\section{Preliminaries}
\label{sec:2}
This section introduces the hierarchical control architecture, the mobile robot model, and key safety notions based on Control Barrier Functions (CBFs). We then formalize the formation control problem and state the main assumptions used throughout the paper.

\subsection{Mobile Robot Model}\label{sec:robot_model}

We consider $N$ mobile robots in the plane. The configuration of robot
$i \in \{1,\dots,N\}$ is described by its position
$p_i = [p_{x,i}, p_{y,i}]^\top \in \mathbb{R}^2$, heading angle
$\theta_i \in \mathbb{R}$, and forward speed $v_i \in \mathbb{R}$. The
state and control input are
\begin{equation}
    \boldsymbol{x}_i
    =
    \begin{bmatrix}
        p_{x,i} \\ p_{y,i} \\ \theta_i \\ v_i
    \end{bmatrix}
    \in \mathbb{R}^4,
    \qquad
    \boldsymbol{u}_i
    =
    \begin{bmatrix}
        a_i \\ \omega_i
    \end{bmatrix}
    \in \mathbb{R}^2,
\end{equation}
where $a_i$ and $\omega_i$ denote the longitudinal acceleration and
angular velocity, respectively.

Each robot follows unicycle-type kinematics with speed dynamics:
\begin{equation}
\label{eq:unicycle}
\begin{aligned}
    \dot{p}_{x,i} &= v_i \cos\theta_i,\\
    \dot{p}_{y,i} &= v_i \sin\theta_i,\\
    \dot{\theta}_i &= \omega_i,\\
    \dot{v}_i      &= a_i,
\end{aligned}
\end{equation}
subject to input and speed bounds
$a_i \in [a_{\min}, a_{\max}]$,
$\omega_i \in [\omega_{\min}, \omega_{\max}]$,
$v_i \in [v_{\min}, v_{\max}]$.

The dynamics \eqref{eq:unicycle} can be written in control-affine form
\begin{equation}
\label{eq:control_affine}
    \dot{\boldsymbol{x}}_i
    = f(\boldsymbol{x}_i) + g(\boldsymbol{x}_i)\,\boldsymbol{u}_i,
\end{equation}
with
\begin{equation}
\label{eq:f_g}
    f(\boldsymbol{x}_i) =
    \begin{bmatrix}
        v_i \cos\theta_i \\
        v_i \sin\theta_i \\
        0 \\
        0
    \end{bmatrix},
    \qquad
    g(\boldsymbol{x}_i) =
    \begin{bmatrix}
        0 & 0 \\
        0 & 0 \\
        0 & 1 \\
        1 & 0
    \end{bmatrix}.
\end{equation}

\begin{definition}
The vector fields $f, g$ are locally Lipschitz on $\mathbb{X}_i \subseteq \mathbb{R}^4$ if, for any $\boldsymbol{x}_i \in \mathbb{X}_i$, there exist $M>0$ and $\delta>0$ such that for all $\boldsymbol{x}_i,\boldsymbol{x}_i' \in \mathbb{X}_i$ with $\|\boldsymbol{x}_i-\boldsymbol{x}_i'\|\le\delta$,
\[
\|f(\boldsymbol{x}_i)-f(\boldsymbol{x}_i')\|\le M\|\boldsymbol{x}_i-\boldsymbol{x}_i'\|,\quad
\|g(\boldsymbol{x}_i)-g(\boldsymbol{x}_i')\|\le M\|\boldsymbol{x}_i-\boldsymbol{x}_i'\|.
\]
\end{definition}

\subsection{Problem Statement and Safety Notions}

We consider a multi-robot formation control problem in an unknown environment, as illustrated in Fig.~\ref{fig2}. The goal is to design distributed control inputs $\boldsymbol{u}_i$ that drive all robots from their initial states to designated goal states while: (i) maintaining a desired formation; (ii) avoiding collisions with obstacles and neighboring robots; and (iii) respecting state and input constraints.

\begin{figure}[t]
    \centering
    \includegraphics[width=1\columnwidth]{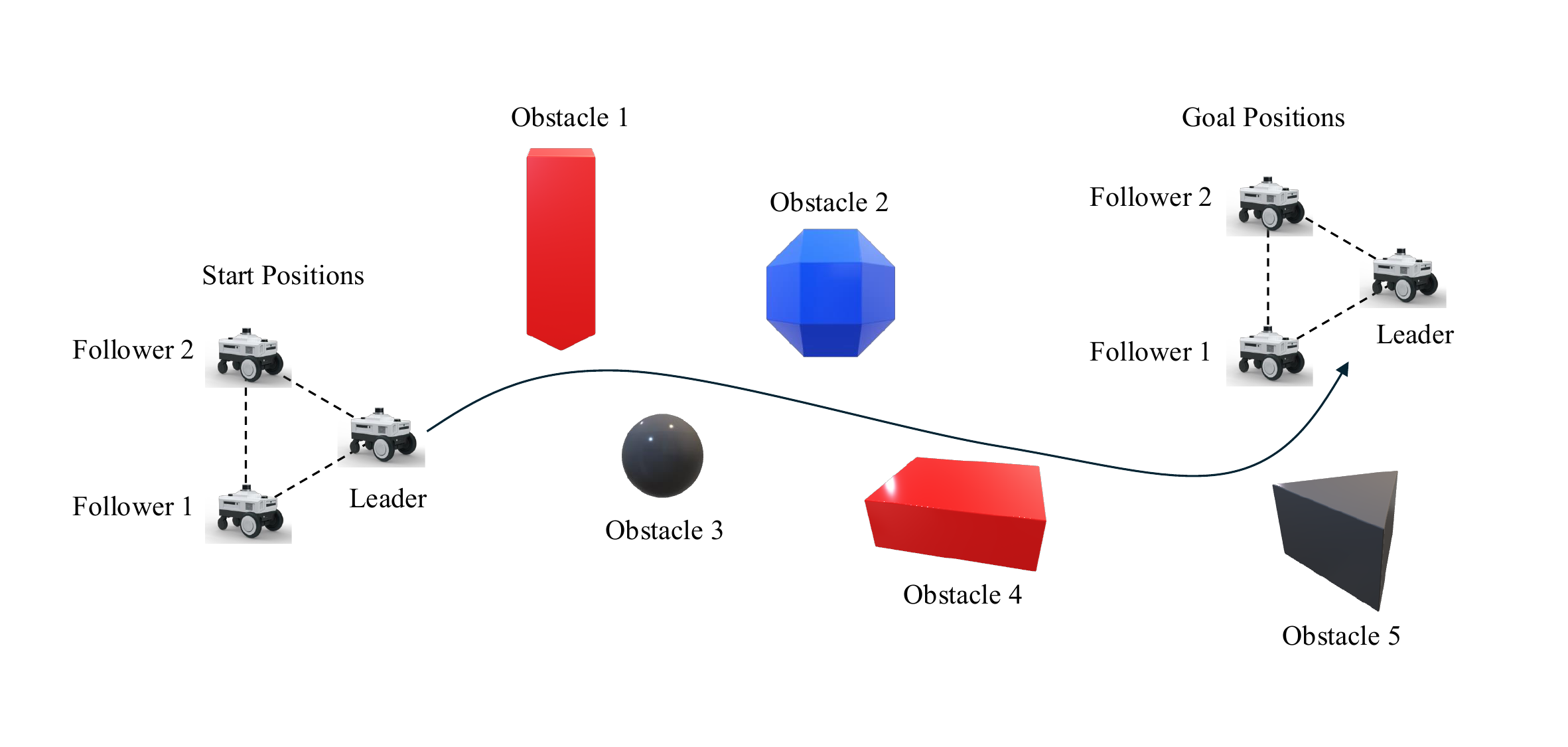}
    \vspace{-20pt}
    \caption{Cooperative formation and safety-aware navigation for multiple mobile robots.}
    \vspace{-0.7cm}
    \label{fig2}
\end{figure}

For each agent $i$, let $\mathbb{X}_{\text{safe},i} \subseteq \mathbb{X}_i$ and
$\mathbb{U}_{\text{safe},i} \subseteq \mathbb{U}_i$ denote the safe state and input sets. Safe control means that for all $t \ge 0$,
\[
\boldsymbol{x}_i(t) \in \mathbb{X}_{\text{safe},i}, \quad
\boldsymbol{u}_i(t) \in \mathbb{U}_{\text{safe},i},
\]
and that a minimum safety distance $s>0$ is maintained between robot $i$ and any obstacle or neighbor, i.e.,
$d(\boldsymbol{x}_i,\boldsymbol{o}_i)\ge s$ for all $i$, where $\boldsymbol{o}_i$ collects local obstacle/neighbor information perceived by robot $i$.

\begin{definition}
A continuous function $\alpha:[0,a) \rightarrow [0,\infty)$ is class-$\kappa$ if it is strictly increasing and satisfies $\alpha(0)=0$. An extended class-$\kappa$ function $\beta:\mathbb{R}\rightarrow\mathbb{R}$ is strictly increasing and satisfies $\beta(0)=0$.
\end{definition}

\begin{definition}
The Lie derivative of a scalar function $\eta(\boldsymbol{x}_i)$ along a vector field $\xi(\boldsymbol{x}_i)$ is
\[
L_{\xi}\eta(\boldsymbol{x}_i)
= \frac{\partial \eta(\boldsymbol{x}_i)}{\partial \boldsymbol{x}_i}\,\xi(\boldsymbol{x}_i).
\]
\end{definition}

\begin{definition}[CBF \cite{ames2016control}]
A continuously differentiable function $h:\mathbb{X}_i \rightarrow \mathbb{R}$ is a Control Barrier Function (CBF) if it induces a safe set
$
S := \{\boldsymbol{x}_i \mid h(\boldsymbol{x}_i)\ge 0\}
$
and there exists an extended class-$\kappa$ function $\alpha(\cdot)$ such that, for the dynamics \eqref{eq:control_affine},
\[
\sup_{\boldsymbol{u}_i\in\mathbb{U}_i}
\big[L_f h(\boldsymbol{x}_i) + L_g h(\boldsymbol{x}_i)\boldsymbol{u}_i\big]
\;\ge\; -\alpha\big(h(\boldsymbol{x}_i)\big),
\quad \forall \boldsymbol{x}_i\in S.
\]
If this condition holds, $S$ is forward invariant.
\end{definition}

We use decentralized CBFs to encode local safety for each robot.

\begin{definition}[Decentralized safety]\label{definition:cbfloss}
A decentralized CBF for agent $i$ is a function
$h_i:\mathbb{X}_i \times \mathbb{O}_i \rightarrow \mathbb{R}$, where $\boldsymbol{o}_i$ denotes local obstacle and neighbor information. It defines the safe set
\[
S_i = \{\boldsymbol{x}_i \mid h_i(\boldsymbol{x}_i,\boldsymbol{o}_i)\ge 0\},
\]
with boundary $\partial S_i = \{\boldsymbol{x}_i \mid h_i(\boldsymbol{x}_i,\boldsymbol{o}_i)=0\}$ and interior
$\operatorname{Int}(S_i) = \{\boldsymbol{x}_i \mid h_i(\boldsymbol{x}_i,\boldsymbol{o}_i)>0\}$.
If there exists an extended class-$\kappa$ function $\alpha$ such that
\[
\sup_{\boldsymbol{u}_i\in\mathbb{U}_i}
\big[L_f h_i(\boldsymbol{x}_i,\boldsymbol{o}_i) + L_g h_i(\boldsymbol{x}_i,\boldsymbol{o}_i)\boldsymbol{u}_i\big]
\;\ge\; -\alpha\big(h_i(\boldsymbol{x}_i,\boldsymbol{o}_i)\big)
\]
for all $\boldsymbol{x}_i\in S_i$, then $S_i$ is forward invariant and agent $i$ remains safe. If this holds for all $i$, the entire multi-agent system is safe.
\end{definition}

\textcolor{black}{Since our focus is on designing a fast policy learning algorithm for distributed MPC (DMPC), we make two standard assumptions to ensure well-posedness of formation control and safety.}

\begin{assm}
The communication network is time-invariant and delay-free during formation maintenance and obstacle avoidance. The interaction topology does not change, and neighboring robots can exchange state information instantaneously, so each robot has access to the states of its neighbors.
\end{assm}

\begin{assm}
For each robot $i$, there exists a state-feedback law $\boldsymbol{u}_i(\boldsymbol{x}_i)$ that renders the formation-tracking error stable while keeping the state within the collision-free set for all time; i.e., the error converges and the safe set is forward invariant.

\end{assm}

%% file: section/CONTROLLER_DESIGN.tex
\section{Safe Formation Controller Design}
\label{sec:3}

This section presents the safe formation controller in \textsc{FORMULA}, a distributed framework for multi-robot formation maintenance and safety-aware navigation. \textsc{FORMULA} couples an MPC-CLF layer for formation tracking with neural network–parameterized decentralized CBFs for safety, improving adaptability across team sizes and environment complexity. An event-triggered mechanism resolves deadlocks when the nominal MPC solution conflicts with CBF constraints, ensuring continuous, stable motion. The following subsections detail the MPC-CLF formulation, the learning-based CBF design, and the deadlock-resolution strategy, summarized in Fig.~\ref{fig:framework}.

\subsection{MPC With CLF Constraints}\label{sec:mpc-clf}

We design a distributed MPC scheme with CLF constraints to extend single-robot MPC-CLF techniques to the multi-robot formation setting. The objective is to stabilize the formation tracking error while respecting the dynamics~\eqref{eq:control_affine} and input bounds for each robot $i \in \{1,\dots,N\}$.

To reduce communication and computation, \textsc{FORMULA} adopts a fully distributed structure: each robot exchanges state information only with its neighbors $\mathcal{N}_i$. We define the nominal state for robot $i$ as a convex combination of neighbor and leader states:
\begin{equation}
\hat{\boldsymbol{x}}_i
=
\frac{\displaystyle\sum_{j \neq i} c_{ij}\big(\boldsymbol{x}_j + \boldsymbol{\Delta}_{ij}\big)
+ s_i\big(\boldsymbol{x}_r + \boldsymbol{\Delta}_{ir}\big)}
{\displaystyle\sum_{j \neq i} c_{ij} + s_i},
\label{Eq.nominalstate}
\end{equation}
where $\hat{\boldsymbol{x}}_i = [p_{x,i}, p_{y,i}, \theta_i, v_i]^\top \in \mathbb{R}^4$ is the nominal state of robot $i$; $\boldsymbol{x}_j$ and $\boldsymbol{x}_r$ are the states of neighbor $j$ and the leader, respectively; $c_{ij} \in \{0,1\}$ encodes whether $j \in \mathcal{N}_i$ ($c_{ij}=1$ if $j \in \mathcal{N}_i$, and $c_{ij}=0$ otherwise); and $s_i \in \{0,1\}$ indicates whether robot $i$ receives the leader state. The vectors $\boldsymbol{\Delta}_{ij}$ and $\boldsymbol{\Delta}_{ir}$ specify the desired formation offsets relative to neighbors and the leader.

The local formation-tracking error is defined as
\[
\boldsymbol{e}_{\mathcal{N}_i} := \boldsymbol{x}_i - \hat{\boldsymbol{x}}_i, \label{error}
\]
and we use the quadratic Lyapunov function
$V_i(\boldsymbol{e}_{\mathcal{N}_i}) = \boldsymbol{e}_{\mathcal{N}_i}^\top \boldsymbol{e}_{\mathcal{N}_i}$.
The MPC-CLF problem for robot $i$ is
\begin{equation}
\begin{aligned}
\min_{\hat{\boldsymbol{u}}_i(\cdot)} \quad &  
J_{\mathrm{CLF},i}\big(\hat{\boldsymbol{u}}_i\big)
  = \int_t^{t+T} \|\hat{\boldsymbol{u}}_i(\tau)\|^2 \,\mathrm{d}\tau \\
\text{s.t.} \quad &
\dot{\boldsymbol{x}}_i = f(\boldsymbol{x}_i) + g(\boldsymbol{x}_i)\,\hat{\boldsymbol{u}}_i, \\
& \dot{V}_i\big(\boldsymbol{e}_{\mathcal{N}_i}, \hat{\boldsymbol{u}}_i\big)
  + \beta\,V_i\big(\boldsymbol{e}_{\mathcal{N}_i}\big) \le 0, \\
& v_{\min} \le v_i \le v_{\max}, \\
& \boldsymbol{u}_{\min} \le \hat{\boldsymbol{u}}_i \le \boldsymbol{u}_{\max},
\end{aligned}
\label{Eq. 3}
\end{equation}
where $\beta>0$ is a design constant, $T$ is the prediction horizon, and 
$\boldsymbol{u}_{\min}, \boldsymbol{u}_{\max}$ denote the input bounds consistent with the limits in Section~\ref{sec:robot_model}. The first element of the optimal sequence $\hat{\boldsymbol{u}}_i^\ast$ is applied as the actual input $\boldsymbol{u}_i$.

We recall the CLF notion used in \textsc{FORMULA}.

\begin{definition}[\cite{minniti2021adaptive}]\label{definition:clf}
A continuously differentiable function $V:\mathbb{R}^4 \rightarrow \mathbb{R}$ is a CLF for robot $i$ if there exist extended class-$\kappa$ functions $\beta_1,\beta_2,\beta_3$ such that for all $\boldsymbol{e}_{\mathcal{N}_i} \in \mathbb{R}^4$:
\begin{itemize}
    \item \emph{Positive definiteness:}
    \vspace{-4pt}
    \[
    \beta_1\big(\|\boldsymbol{e}_{\mathcal{N}_i}\|\big)
    \le V(\boldsymbol{e}_{\mathcal{N}_i})
    \le \beta_2\big(\|\boldsymbol{e}_{\mathcal{N}_i}\|\big),
    \]
    where $V(\boldsymbol{e}_{\mathcal{N}_i})=0$ iff $\boldsymbol{e}_{\mathcal{N}_i}=0$, and $V(\boldsymbol{e}_{\mathcal{N}_i})>0$ otherwise.
    \item \emph{Exponential decrease:}
    \vspace{-4pt}
    \[
    \inf_{\boldsymbol{u}_i \in \mathbb{U}_i}
    \big[
        L_f V(\boldsymbol{e}_{\mathcal{N}_i})
        + L_g V(\boldsymbol{e}_{\mathcal{N}_i})\,\boldsymbol{u}_i
    \big]
    \le -\beta_3\big(V(\boldsymbol{e}_{\mathcal{N}_i})\big),
    \]
    ensuring that $\boldsymbol{e}_{\mathcal{N}_i}$ converges exponentially to zero under a suitable feedback policy.
\end{itemize}
\end{definition}

Differentiating the nominal state~\eqref{Eq.nominalstate} and using the control-affine dynamics~\eqref{eq:control_affine} yields
\begin{equation}\label{eq:nominaldynamics}
\dot{\hat{\boldsymbol{x}}}_i
=
\frac{\displaystyle\sum_{j \neq i} c_{ij}\big(f(\boldsymbol{x}_j)+g(\boldsymbol{x}_j)\boldsymbol{u}_j\big)
+ s_i\big(f(\boldsymbol{x}_r)+g(\boldsymbol{x}_r)\boldsymbol{u}_r\big)}
{\displaystyle\sum_{j \neq i} c_{ij} + s_i}.
\end{equation}
The time derivative of $V_i(\boldsymbol{e}_{\mathcal{N}_i})$ follows as
\begin{equation}
\begin{aligned}
\dot{V}_i\big(\boldsymbol{e}_{\mathcal{N}_i}, \boldsymbol{u}_i\big)
&= 2\,\boldsymbol{e}_{\mathcal{N}_i}^\top\big(\dot{\boldsymbol{x}}_i - \dot{\hat{\boldsymbol{x}}}_i\big)\\
&= \frac{\partial V_i(\boldsymbol{x}_i-\hat{\boldsymbol{x}}_i)}{\partial \boldsymbol{x}_i}\dot{\boldsymbol{x}}_i
 + \frac{\partial V_i(\boldsymbol{x}_i-\hat{\boldsymbol{x}}_i)}{\partial \hat{\boldsymbol{x}}_i}\dot{\hat{\boldsymbol{x}}}_i \\
&= 2(\boldsymbol{x}_i-\hat{\boldsymbol{x}}_i)^\top
\big(f(\boldsymbol{x}_i)+g(\boldsymbol{x}_i)\boldsymbol{u}_i-\dot{\hat{\boldsymbol{x}}}_i\big),
\end{aligned}
\label{Eq. 4}
\end{equation}
which connects the CLF decrease condition to the control input and the nominal formation dynamics defined by neighbors and the leader.

\begin{lemma}
    If the optimal MPC-CLF cost $J_{\mathrm{CLF},i}(\hat{\boldsymbol{u}}_i^\ast)$ converges to zero for robot $i$, then the CLF constraint becomes inactive, and the closed-loop dynamics~\eqref{eq:control_affine}--\eqref{eq:nominaldynamics} drive the formation error $\boldsymbol{e}_{\mathcal{N}_i}$ to zero, guaranteeing formation stability without additional control effort.
\end{lemma}

\begin{proof}
    The result follows from standard MPC-CLF optimality arguments: when $J_{\mathrm{CLF},i}(\hat{\boldsymbol{u}}_i^\ast)\to 0$, the optimal input sequence tends to $\hat{\boldsymbol{u}}_i^\ast \to \boldsymbol{0}$, and the KKT conditions for~\eqref{Eq. 3} imply that the CLF constraint is satisfied with equality only when $\boldsymbol{x}_i = \hat{\boldsymbol{x}}_i$. Thus $V_i(\boldsymbol{e}_{\mathcal{N}_i}) \to 0$, and the formation error converges to zero. \qedhere
\end{proof}

\begin{figure*}[t]
	\centering
	\includegraphics[width=\textwidth]{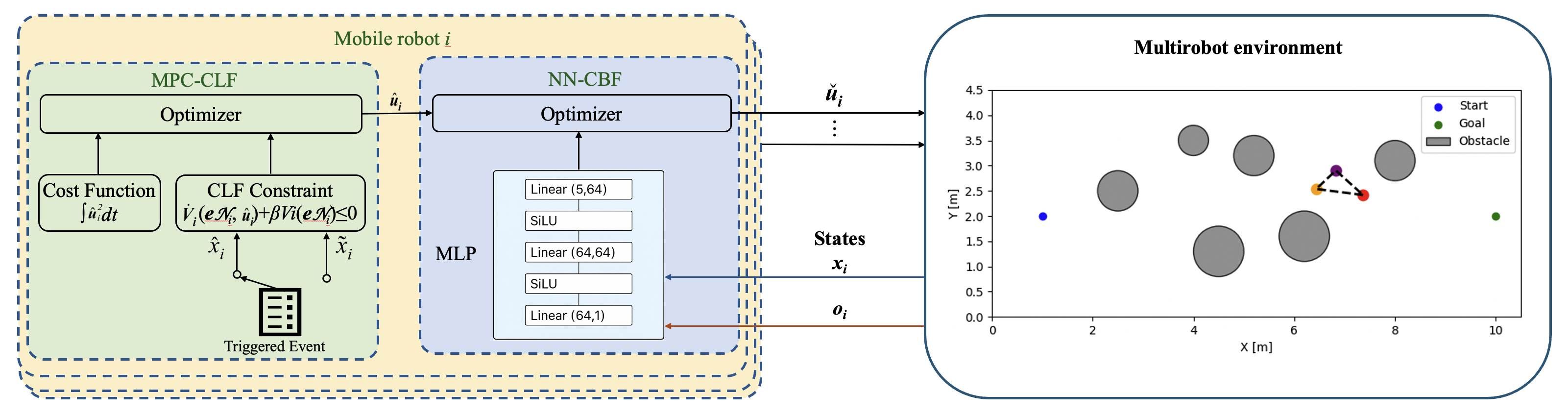}
    \vspace{-0.6cm}
	\caption{Overview of the proposed \textsc{FORMULA} framework. Left: each robot runs a distributed MPC-CLF optimizer for formation stability and a NN--CBF module for safety, with an event-triggered deadlock-resolution mechanism. Right: one of the multi-robot simulation environments.}
    \label{fig:framework}
    \vspace{-0.5cm}
\end{figure*}

\subsection{Neural Network-Based Control Barrier Function}\label{sec:nn-cbf}
Beyond formation tracking, \textsc{FORMULA} must guarantee safety by avoiding collisions with obstacles and neighboring robots. To this end, we introduce a Neural Network-based CBF (NN-CBF) for each agent that adapts the decentralized CBF $h_i$ in Definition~\ref{definition:cbfloss} to varying robot densities and environment clutter while preserving the safety condition.

Let $h_i^{\theta_i}:\mathbb{X}_i \times \mathbb{O}_i \rightarrow \mathbb{R}$ denote the parametric CBF for robot $i$, with neural-network parameters $\theta_i$. Based on Definition~\ref{definition:cbfloss}, we define the empirical CBF loss as
\begin{equation}
\mathbf{L}_{\mathrm{CBF}} \;=\; \sum_{i=1}^N L^i_{\mathrm{CBF}},
\label{Eq.tatalcbf}
\end{equation}
where the per-agent loss $L^i_{\mathrm{CBF}}$ is
\begin{equation}
\begin{aligned}
L^i_{\mathrm{CBF}}(\theta_i)
&=\sum_{\boldsymbol{x}_i \in \mathbb{X}_{\text{safe},i}}
    \max\!\big(0,\, \gamma - h_i^{\theta_i}(\boldsymbol{x}_i,\boldsymbol{o}_i)\big) \\
&\quad + \sum_{\boldsymbol{x}_i \in \mathbb{X}_{\text{unsafe},i}}
    \max\!\big(0,\, \gamma + h_i^{\theta_i}(\boldsymbol{x}_i,\boldsymbol{o}_i)\big) \\
&\quad + \sum_{\boldsymbol{x}_i \in \mathbb{X}_h}
    \max\Big(0,\, \gamma
      - L_f h_i^{\theta_i}(\boldsymbol{x}_i,\boldsymbol{o}_i)
      - L_g h_i^{\theta_i}(\boldsymbol{x}_i,\boldsymbol{o}_i)\,\boldsymbol{u}_i \\
&\qquad\qquad\qquad\qquad\qquad
      - \alpha\big(h_i^{\theta_i}(\boldsymbol{x}_i,\boldsymbol{o}_i)\big)\Big).
\end{aligned}
\label{Eq.lcbf}
\end{equation}
Here $\mathbb{X}_{\text{safe},i}$ and $\mathbb{X}_{\text{unsafe},i}$ are sampled safe and unsafe state sets for agent $i$, respectively, $\mathbb{X}_h$ is the set of states where the CBF inequality is evaluated, $\boldsymbol{o}_i \in \mathbb{O}_i$ collects local obstacle/neighbor information, and $\alpha(\cdot)$ is an extended class-$\kappa$ function as in Definition~\ref{definition:cbfloss}. The constant $\gamma>0$ is a safety margin (we use $\gamma = 10^{-3}$ in our implementation). The first two terms in~\eqref{Eq.lcbf} enforce correct classification of safe versus unsafe states, while the third term penalizes violations of the CBF forward-invariance condition, thus encouraging $S_i$ to be forward invariant under the dynamics~\eqref{eq:control_affine}.

The MPC-CLF controller in \textsc{FORMULA} yields a nominal input $\hat{\boldsymbol{u}}_i$ from~\eqref{Eq. 3}. If we applied only the NN-CBF constraint without regularization, the safest behavior in cluttered regions would often be to stop, causing followers to freeze in front of obstacles. To avoid such overly conservative behavior and keep the refined control close to the MPC-CLF solution, we introduce a regularization loss
\[
L^i_u \;=\; \big\|\check{\boldsymbol{u}}_i - \hat{\boldsymbol{u}}_i\big\|^2,
\]
where $\check{\boldsymbol{u}}_i$ is the refined input that satisfies the NN-CBF constraint. The total training loss for the CBF network is
\[
L \;=\; \sum_{i=1}^N \Big( L^i_{\mathrm{CBF}} + \sigma\,L^i_u \Big),
\]
with $\sigma>0$ weighting safety against control performance. In operation, \textsc{FORMULA} uses $\check{\boldsymbol{u}}_i$ as the applied input, thereby maintaining formation through the MPC-CLF layer while enforcing obstacle and inter-agent safety via NN-CBF.

\subsection{Deadlock Resolution} \label{sec:deadlock}
In dense or tightly constrained environments, the safety constraints enforced by NN-CBFs may conflict with the MPC-CLF objective, yielding near-zero or infeasible inputs and causing the team to stall. This effect is particularly severe in narrow passages, where conservative CBF constraints can block forward motion even though the nominal MPC command would be dynamically feasible.

Within \textsc{FORMULA}, we address this issue using an event-triggered deadlock resolver that perturbs the nominal formation reference, rather than directly modifying the low-level control input. By adjusting the nominal state used by MPC-CLF, we relax the conflict with the CBF constraints while preserving Lyapunov-based stability.

For each robot $i$, we define a deadlock indicator
\begin{equation}
E_i \;=\; \operatorname{sign}\!\big(|J_{\mathrm{CLF},i} - L^i_{\mathrm{CBF}}| + |v_i|\big)
      \;-\; \operatorname{sign}\!\big(J_{\mathrm{CLF},i}\big),
\end{equation}
where $J_{\mathrm{CLF},i}$ is the current MPC-CLF cost for agent $i$, $L^i_{\mathrm{CBF}}$ is the instantaneous CBF violation measure from~\eqref{Eq.lcbf}, and $v_i$ is the forward speed in the state $\boldsymbol{x}_i$. Intuitively, a deadlock is declared when safety penalties dominate while the robot’s speed remains small, i.e., when $E_i < 0$.

When $E_i < 0$, we modify the nominal state~\eqref{Eq.nominalstate} via a bounded, state-aligned transformation:
\begin{equation}
\tilde{\boldsymbol{x}}_i
\;=\;
T\!\big(\hat{\boldsymbol{x}}_i - \boldsymbol{x}_i\big) + \boldsymbol{x}_i,
\end{equation}
where $T \in \mathbb{R}^{4 \times 4}$ (since $\boldsymbol{x}_i \in \mathbb{R}^4$) applies a structured perturbation that biases the nominal formation reference away from configurations that cause CBF infeasibility. The updated nominal state $\tilde{\boldsymbol{x}}_i$ is then used in place of $\hat{\boldsymbol{x}}_i$ when recomputing the MPC-CLF input, producing a new nominal command $\hat{\boldsymbol{u}}_i$ and, after NN-CBF refinement, a nonzero $\check{\boldsymbol{u}}_i$ that breaks the stalemate.

The deadlock resolver is triggered only when $E_i<0$ and is disabled once the conflict between MPC-CLF and NN-CBF relaxes (i.e., $E_i \ge 0$ again). In this way, \textsc{FORMULA} maintains the stability guarantees while ensuring continuous, safety-aware motion in cluttered multi-robot environments.

%% file: section/Experiments.tex
\section{Experiments and Results}
\subsection{Simulation Environment}

In the simulation environment, we evaluate teams of planar mobile robots evolving under the dynamics~\eqref{eq:control_affine} in a 2D $10 m \times 4.5 m$ cluttered workspace while preserving a specified formation, as illustrated in Fig.~\ref{fig:framework} (Right). Each robot
was modeled as a point mass for planning purposes. A designated leader and multiple followers maintain relative offsets $\boldsymbol{\Delta}_{ij}, \boldsymbol{\Delta}_{ir}$ in the leader’s body frame, subject to minimum separation constraints between agents and circular obstacles, consistent with the safety notion in Definition~\ref{definition:cbfloss}.

For the smallest team, we deploy one leader and two followers in a triangular formation. The leader tracks a straight nominal path from a start to a goal position, and followers regulate their formation error $\boldsymbol{e}_{\mathcal{N}_i} = \boldsymbol{x}_i - \hat{\boldsymbol{x}}_i$ around the desired offsets. Circular obstacles are positioned such that the leader’s nominal path passes close to them, forcing the team to compress and then re-expand the formation while maintaining safety. To assess scalability, we additionally test formations with $4$ and $8$ followers in more cluttered environments, with higher obstacle density and tighter passages, as shown in Fig.~\ref{fig7}.

\begin{figure}
	\centering
	\includegraphics[width=\columnwidth]{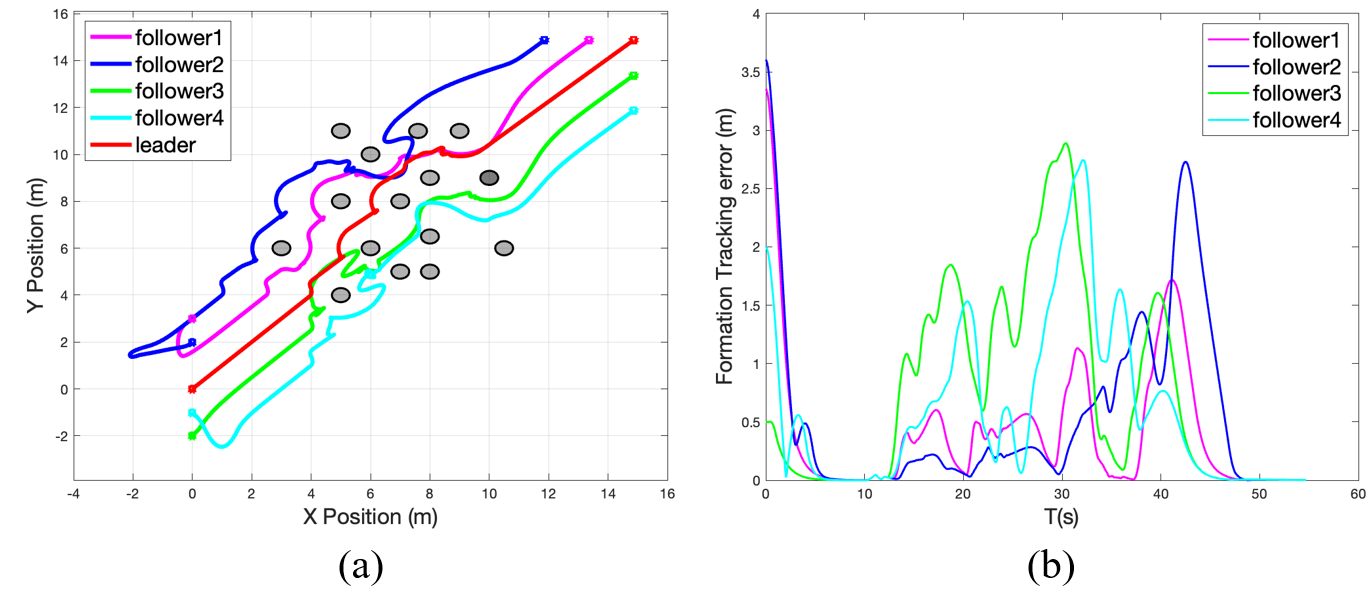}
    \vspace{-0.8cm}
	\caption{Formation control in obstacle-dense environments. (a) Planar trajectories of the leader (red) and followers (magenta, blue, green, cyan) among randomly placed obstacles (gray circles). (b) Formation tracking error $\|\boldsymbol{e}_{\mathcal{N}_i}\|$ over time for all followers, showing a temporary increase during obstacle negotiation followed by convergence as the formation re-stabilizes.}
	\label{fig7}
 \vspace{-18pt}
\end{figure}

\subsection{Baseline Algorithms}

We compare \textsc{FORMULA} against three baselines:

(i) \textbf{Artificial Potential Fields (APF) \cite{zhou2018agile}:}  
Each follower is attracted to its desired formation position and repelled from obstacles and neighboring agents via hand-crafted potential fields. This approach is computationally light but prone to local minima and oscillations in cluttered, multi-agent settings.

(ii) \textbf{MPC-CBF \cite{grandia2020nonlinear}:}  
A model predictive controller with CBF constraints that directly embeds safety into a finite-horizon optimization. The discrete-time dynamics are
\[
\boldsymbol{x}_{k+1}
= \boldsymbol{x}_k + \big(f(\boldsymbol{x}_k) + g(\boldsymbol{x}_k)\,\boldsymbol{u}_k\big)\Delta t,
\]
and, at each step, a quadratic cost on state tracking and control effort is minimized subject to state/input bounds and CBF-based safety constraints.

(iii) \textbf{CLF+CBF-QP (analytic) \cite{desai2022clf}:}  
A CLF-based formation controller augmented with an analytic distance CBF. For each robot, we solve a small QP that finds $\boldsymbol{u}_i$ close to the nominal CLF control while satisfying continuous-time CBF constraints with respect to circular obstacles. Let $p_i \in \mathbb{R}^2$ be the position of robot $i$ and $p_{\text{obs}}$ the obstacle center. The barrier is
\begin{equation}
    h(p_i) = \|p_i - p_{\text{obs}}\|^2 - (R_{\text{rob}} + R_{\text{obs}} + s)^2,
\end{equation}
and the QP enforces the standard CBF inequality
\begin{equation}
    \nabla h(p_i)^\top \dot{p}_i \;\ge\; -\alpha\, h(p_i),
\end{equation}
for a class-$\kappa$ constant $\alpha>0$. This serves as a fully model-based safety baseline.

\subsection{Implementation Details}
\label{sec:implementation_details}

\textsc{FORMULA} retains the MPC--CLF formation layer in Section~\ref{sec:mpc-clf} and replaces the analytic distance CBF with the NN-CBF $h_i^{\theta_i}(\boldsymbol{x}_i,\boldsymbol{o}_i)$ from Section~\ref{sec:nn-cbf}. The NN-CBF receives as input the robot position $\boldsymbol{x}_i \in \mathbb{R}^2$ and obstacle parameters $\boldsymbol{o}_i = (p_x,p_y,r) \in \mathbb{R}^3$, concatenated into a five-dimensional vector, and is implemented as a lightweight multilayer perceptron with two hidden layers of width 64 with SiLU activations, followed by a linear output layer that produces a scalar barrier value. At run time, we obtain the nominal control $\hat{\boldsymbol{u}}_i$ from the MPC--CLF problem~\eqref{Eq. 3}, then linearize the learned CBF around the current state to form constraints $a_k^\top \boldsymbol{u}_i \ge b_k$ with
$a_k = \nabla_{\boldsymbol{x}_i} h_i^{\theta_i}(\boldsymbol{x}_i,\boldsymbol{o}_{i,k})$ and $b_k = -\alpha\,h_i^{\theta_i}(\boldsymbol{x}_i,\boldsymbol{o}_{i,k})$. The final control is obtained by a small QP projection,
\[
\check{\boldsymbol{u}}_i
= \arg\min_{\boldsymbol{u}_i}
\big\|\boldsymbol{u}_i - \hat{\boldsymbol{u}}_i\big\|^2
\quad \text{s.t.}\;
a_k^\top \boldsymbol{u}_i \ge b_k,\;\forall k,
\]
which enforces the NN-CBF constraint~\eqref{Eq.lcbf} while staying close to the MPC--CLF solution.

Training proceeds in two stages (using Adam with learning rate $10^{-3}$ and mini-batches of size $512$). In Stage~1, we pre-train on uniformly sampled $(\boldsymbol{x}_i,\boldsymbol{o}_i)$ over the workspace to imitate the analytic distance barrier, using the loss $L$ in Section~\ref{sec:nn-cbf} augmented with a regression term $\mathbb{E}[(h_i^{\theta_i} - h_{\text{true}})^2]$ toward the analytic CBF value $h_{\text{true}}$. In Stage~2, we fine-tune on trajectories generated by the analytic CLF+CBF-QP controller for teams of $2$, $4$, and $8$ followers, using the same loss to focus learning on states visited under realistic formation maneuvers and near the safe-set boundary.

\subsection{Metrics}

\textbf{Safety rate.}  
We define the safety rate as the fraction of simulation time steps at which all inter-agent and agent–obstacle distances satisfy the safety condition $d(\boldsymbol{x}_i,\boldsymbol{o}_i)\ge s$ for all $i$ (cf. Definition~\ref{definition:cbfloss}). From Table \ref{tab:controller_results}, APF shows low safety across all team sizes, MPC-CBF improves safety but degrades as formations become denser, and analytic CLF+CBF-QP yields consistently conservative behavior. \textsc{FORMULA} consistently outperforms the analytic CBF in safety for teams of varying sizes.

\textbf{Scalability.}  
We test systems with $2$, $4$, and $8$ followers. As $N$ increases, APF quickly becomes unreliable due to local minima and congestion, and MPC-CBF’s safety rate also drops under heavy crowding. In contrast, CBF-based controllers scale more gracefully: analytic CLF+CBF-QP and FORMULA maintain high safety and efficient formations for larger teams compared to non-CBF baselines.

\textbf{Formation error.}  
Formation accuracy is measured as the time-averaged norm of the local formation error over all followers and time. APF and MPC-CBF incur larger errors due to local minima (APF) and aggressive obstacle-avoidance maneuvers (MPC-CBF). Analytic CLF+CBF-QP reduces this error, but remains conservative near obstacles. \textsc{FORMULA} consistently achieves the lowest formation error across team sizes, improving upon CLF+CBF-QP while maintaining safety.

\begin{table}[t]
    \centering
    \scriptsize
    \setlength{\tabcolsep}{3pt}
    \renewcommand{\arraystretch}{1.1}
    \caption{Safety rate, average formation error, and average minimum distance for different controllers and team sizes. Each entry is averaged over one representative rollout.}
    \label{tab:controller_results}
    \begin{tabular}{c l c c c}
        \toprule
        Followers & Method        & Safety rate $\uparrow$ & Avg. form. err. [m] $\downarrow$ & Avg. min. dist. [m] $\uparrow$ \\
        \midrule
        \multirow{4}{*}{2} 
          & APF                & 0.297 & 1.090 & 0.059 \\
          & MPC-CBF            & 0.755 & 1.129 & 0.417 \\
          & CLF+CBF-QP         & 0.717 & 0.787 & 0.409 \\
          & \textbf{Proposed}      & \textbf{0.863} & \textbf{0.705} & \textbf{0.490} \\
        \midrule
        \multirow{4}{*}{4} 
          & APF                & 0.220 & 0.925 & 0.078 \\
          & MPC-CBF            & 0.456 & 1.045 & 0.353 \\
          & CLF+CBF-QP         & 0.637 & 0.823 & 0.429 \\
          & \textbf{Proposed}      & \textbf{0.788} & \textbf{0.715} & \textbf{0.463} \\
        \midrule 
        \multirow{4}{*}{8} 
          & APF                & 0.169 & 0.896 & 0.071 \\
          & MPC-CBF            & 0.361 & 1.041 & 0.334 \\
          & CLF+CBF-QP         & 0.479 & 0.911 & \textbf{0.402} \\
          & \textbf{Proposed}      & \textbf{0.716} & \textbf{0.844} & 0.389 \\
        \bottomrule 
    \end{tabular}
\end{table}

\subsection{Deadlock Resolution in Practice}

We further evaluate the event-triggered deadlock resolver in \textsc{FORMULA} (Section~\ref{sec:deadlock}) using an intersection scenario with four robots whose desired paths cross near the origin. Without deadlock resolution, NN-CBF constraints and formation objectives can conflict, causing robots to stall in front of the intersection. With the deadlock indicator $E_i$ and nominal-state transformation $\tilde{\boldsymbol{x}}_i$ activated when $E_i<0$, robots locally adjust their nominal formation references and generate nonzero refined inputs $\check{\boldsymbol{u}}_i$. As shown in Fig.~\ref{fig8}, all robots negotiate the intersection safely, briefly deforming the formation to satisfy CBF constraints and then reestablishing the desired formation once the conflict region is cleared.

\begin{figure}[t]
	\centering
	\includegraphics[width=0.65\columnwidth]{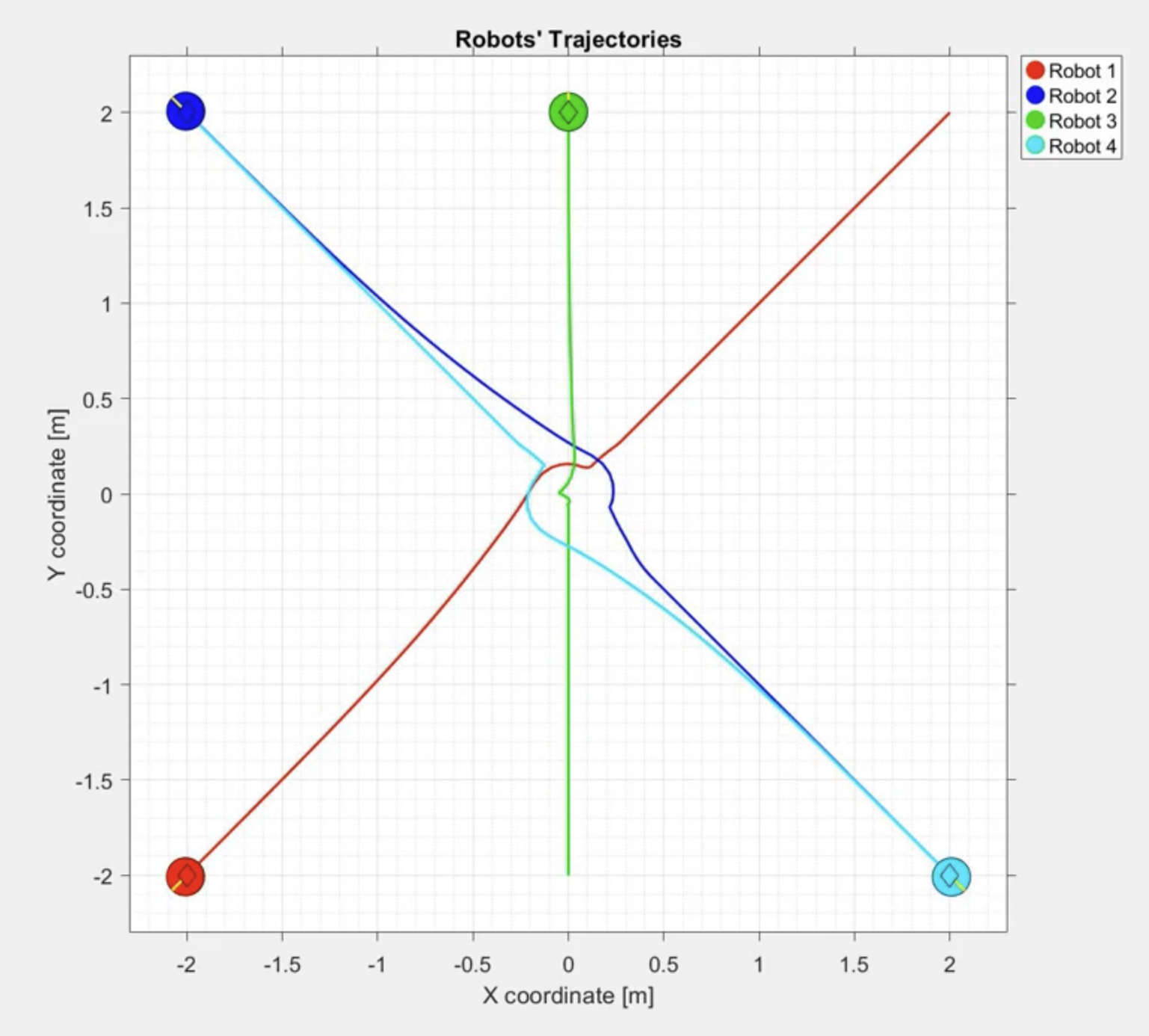}
    \vspace{-10pt}
	\caption{Deadlock resolution in \textsc{FORMULA}. Colored trajectories show four robots safely crossing at an intersection near the origin. The event-triggered mechanism perturbs the nominal states $\hat{\boldsymbol{x}}_i$ when $E_i<0$, breaking potential deadlocks while maintaining safety and formation coherence.}
	\label{fig8}
 \vspace{-18pt}
\end{figure}

%% file: section/Conclusion_and_Future_Work.tex
\section{Conclusion and Future Work}
\label{sec:5}

In this paper, we presented \textsc{FORMULA}, a distributed safe formation-control framework that combines an MPC--CLF layer for formation maintenance with neural network-based decentralized CBFs for reducing collisions. Operating on the control-affine robot dynamics and neighbor-based nominal states, \textsc{FORMULA} preserves formation integrity, adapts its safety margins to varying team sizes and obstacle densities, and resolves safety-induced stalls through an event-triggered deadlock mechanism. Simulations in cluttered 2D environments with robots of up to $8$ followers demonstrate high safety rates and low formation errors, outperforming APF, MPC-CBF, and analytic CLF+CBF-QP baselines. Future work will validate \textsc{FORMULA} on physical robot platforms, extend it to heterogeneous and higher-dimensional multi-robot teams, and enhance the learning pipeline with adaptive optimization and reinforcement learning for online refinement of the NN-CBFs in safety-critical applications such as search-and-rescue and automated warehouses.